\documentclass[sigconf, dvipsnames]{acmart}
\usepackage{amsmath,amsfonts,bm}

\def\eqref#1{equation~\ref{#1}}

\def\1{\bm{1}}

\def\vh{{\bm{h}}}

\DeclareMathAlphabet{\mathsfit}{\encodingdefault}{\sfdefault}{m}{sl}
\SetMathAlphabet{\mathsfit}{bold}{\encodingdefault}{\sfdefault}{bx}{n}

\def\gG{{\mathcal{G}}}
\def\gGc{{\mathcal{G}_c}}

\newcommand{\E}{\mathbb{E}}
\newcommand{\Ls}{\mathcal{L}}
\newcommand{\R}{\mathbb{R}}

\def\V{{\mathcal{V}}}
\def\E{{\mathcal{E}}}

\def\A{{\mathcal{A}}}
\def\R{{\mathcal{R}}}
\def\P{{\mathcal{P}}}
\def\M{{\mathcal{M}}}
\usepackage[T1]{fontenc}    
\usepackage[utf8]{inputenc}
\usepackage{algorithm}
\usepackage{algpseudocode}
\usepackage{amsfonts}       
\usepackage{amsthm}
\usepackage{cases}
\usepackage{color}
\usepackage{enumitem}
\usepackage{fancybox}
\usepackage{graphicx}
\usepackage{hyperref}       
\usepackage{listings}
\usepackage{mathtools}
\usepackage{microtype}
\usepackage{multirow}
\usepackage{url}            
\usepackage{xspace}
\usepackage{subcaption}

\theoremstyle{plain}
\newtheorem{theorem}{Theorem}[section]
\newtheorem{proposition}[theorem]{Proposition}
\newtheorem{lemma}[theorem]{Lemma}

\theoremstyle{definition}
\newtheorem{definition}[theorem]{Definition}

\newtheorem{problem}[theorem]{Problem}

\theoremstyle{remark}

\newcommand{\aminer}{\texttt{AugCitation}\xspace}
\newcommand{\synthetic}{\texttt{UserItemAttr}\xspace}
\newcommand{\method}{\textmd{PaGE-Link}\xspace}
\newcommand{\interpretability}{Connection Interpretability\xspace}
\newcommand{\scalability}{Scalability\xspace}
\newcommand{\heterogeneity}{Heterogeneity\xspace}

\AtBeginDocument{%
  \providecommand\BibTeX{{%
    \normalfont B\kern-0.5em{\scshape i\kern-0.25em b}\kern-0.8em\TeX}}}

\copyrightyear{2023} 
\acmYear{2023} 
\setcopyright{rightsretained} 
\acmConference[WWW '23]{Proceedings of the ACM Web Conference 2023}{May 1--5, 2023}{Austin, TX, USA}
\acmBooktitle{Proceedings of the ACM Web Conference 2023 (WWW '23), May 1--5, 2023, Austin, TX, USA}
\acmDOI{10.1145/3543507.3583511}
\acmISBN{978-1-4503-9416-1/23/04}

\begin{document}

\title{PaGE-Link: Path-based Graph Neural Network Explanation for Heterogeneous Link Prediction}

\author{Shichang Zhang}
\authornote{Work done while being an intern at Amazon Web Services. Code available at: \url{https://github.com/amazon-science/page-link-path-based-gnn-explanation}}
\affiliation{
  \institution{University of California, Los Angeles}
  \country{}
}
\email{shichang@cs.ucla.edu}

\author{Jiani Zhang}
\affiliation{
  \institution{Amazon}
  \country{}
}
\email{zhajiani@amazon.com}

\author{Xiang Song}
\affiliation{
  \institution{Amazon}
  \country{}
}
\email{xiangsx@amazon.com}

\author{Soji Adeshina}
\affiliation{
  \institution{Amazon}
  \country{}
}
\email{adesojia@amazon.com}

\author{Da Zheng}
\affiliation{
  \institution{Amazon}
  \country{}
}
\email{dzzhen@amazon.com}

\author{Christos Faloutsos}
\affiliation{
  \institution{Carnegie Mellon University}
  \institution{Amazon}
  \country{}
}
\email{christos@cs.cmu.edu}

\author{Yizhou Sun}
\affiliation{
  \institution{University of California, Los Angeles}
  \institution{Amazon}
  \country{}
}
\email{yzsun@cs.ucla.edu}

\renewcommand{\shortauthors}{Zhang, Shichang et al.}

\begin{abstract}
Transparency and accountability have become major concerns for black-box machine learning (ML) models. Proper explanations for the model behavior increase model transparency and help researchers develop more accountable models. Graph neural networks (GNN) have recently shown superior performance in many graph ML problems than traditional methods, and explaining them has attracted increased interest. However, GNN explanation for link prediction (LP) is lacking in the literature. LP is an essential GNN task and corresponds to web applications like recommendation and sponsored search on web. Given existing GNN explanation methods only address node/graph-level tasks, we propose \underline{Pa}th-based \underline{G}NN \underline{E}xplanation for heterogeneous \underline{Link} prediction (\textit{\method}) that generates explanations with \textit{connection interpretability}, enjoys model \textit{scalability}, and handles graph \textit{heterogeneity}. Qualitatively, \method can generate explanations as paths connecting a node pair, which naturally captures connections between the two nodes and easily transfer to human-interpretable explanations. Quantitatively, explanations generated by \method improve AUC for recommendation on citation and user-item graphs by \textit{9 - 35\%} and are chosen as better by \textit{78.79\%} of responses in human evaluation.

\end{abstract}



\maketitle

\section{Introduction}\label{sec:introduction}
Transparency and accountability are significant concerns when researchers advance black-box machine learning (ML) models~\cite{fate, fate2}. Good explanations of model behavior improve model transparency. For end users, explanations make them trust the predictions and increase their engagement and satisfaction~\cite{herlocker2000explaining, bilgic2005explaining}. For researchers and developers, explanations enable them to understand the decision-making process and create accountable ML models. Graph Neural Networks (GNNs)~\cite{wu2020comprehensive,zhou2020graph} have recently achieved state-of-the-art performance on many graph ML tasks and attracted increased interest in studying their explainability~\cite{gnnexplainer, pgexplainer, zhang2022gstarx, yuan2022explainability}. However, to our knowledge, GNN explanation for link prediction (LP) is missing in the literature. LP is an essential task of many vital Web applications like recommendation~\cite{zhang2019star, mao2021ultragcn, wu2020graph} and sponsored search~\cite{li2021adsgnn, hao2021ks}. GNNs are widely used to solve LP problems~\cite{zhang2018link,zhu2021neural}, and generating good GNN explanations for LP will benefit these applications, e.g., increasing user satisfaction with recommended items.

\begin{figure}[t]
\centering
\includegraphics[width=\columnwidth]{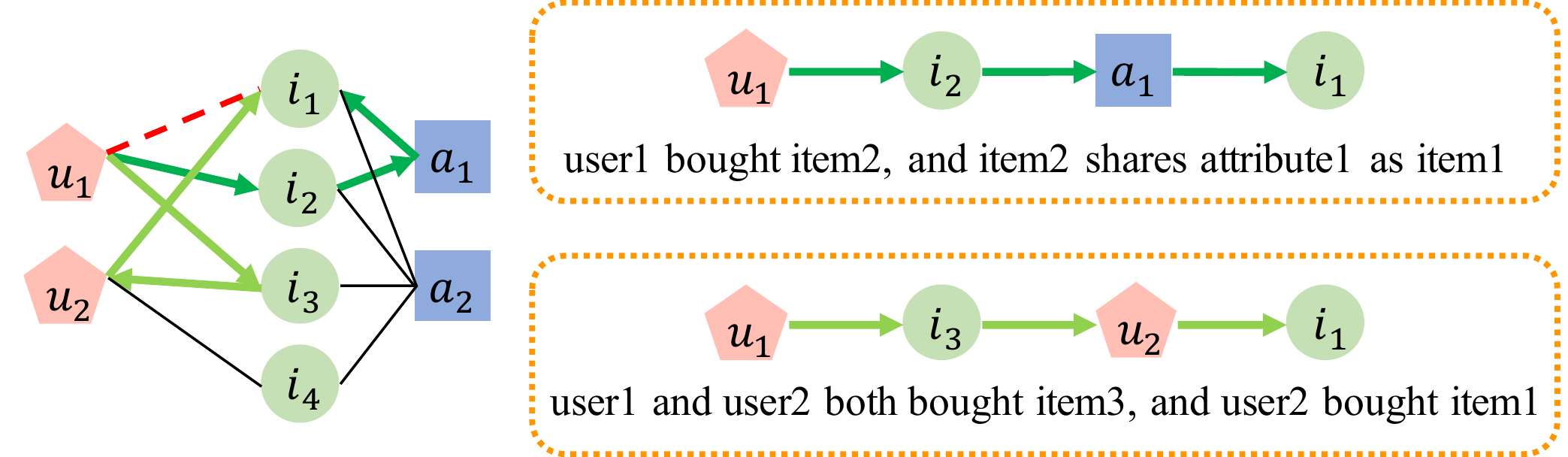}
\caption{Given a GNN model and a predicted link $(u_1, i_1)$ ({\color{red} dashed red}) on a heterogeneous graph of user $u$, item $i$, and attribute $a$ (left). \method generates two path explanations ({\color{ForestGreen}green arrows}). Interpretations illustrated on the right.}
\label{fig:crown_jewel}
\Description{PaGE-Link generates two path explanations on a heterogeneous graph.}
\end{figure}

Existing GNN explanation methods have addressed node/graph-level tasks on homogeneous graphs. Given a data instance, most methods generate an explanation by learning a mask to select an edge-induced subgraph~\cite{gnnexplainer, pgexplainer} or searching over the space of subgraphs~\cite{subgraphx}. However, explaining GNNs for LP is a new and more challenging task. Existing node/graph-level explanation methods do not generalize well to LP for three challenges. 1) \textit{\interpretability}: LP involves a pair of the source node and the target node rather than a single node or graph. Desired interpretable explanations for a predicted link should reveal connections between the node pair. Existing methods generate subgraphs with no format constraints, so they are likely to output subgraphs disconnected from the source, the target, or both. Without revealing connections between the source and the target, these subgraph explanations are hard for humans to interpret and investigate. 2) \textit{\scalability}: For LP, the number of edges involved in GNN computation almost grows from $m$ to ${\sim}2m$ compared to the node prediction task because neighbors of both the source and the target are involved. Since most existing methods consider all (edge-induced) subgraphs, the increased edges will scale the number of subgraph candidates by a factor of $O(2^m)$, which makes finding the optimal subgraph explanation much harder. 3) \textit{\heterogeneity}: Practical LP is often on heterogeneous graphs with rich node and edge types, e.g., a graph for recommendations can have user->buys->item edges and item->has->attribute edges, but existing 
methods only work for homogeneous graphs.

In light of the importance and challenges of GNN explanation for LP, we formulate it as a post hoc and instance-level explanation problem and generate explanations for it in the form of important paths connecting the source node and the target node. Paths have played substantial roles in graph ML and are the core of many non-GNN LP methods~\cite{liben2007link, katz, jeh2002simrank, sun2011pathsim}. Paths as explanations can solve the connection interpretability and scalability challenges. Firstly, paths connecting two nodes naturally explain connections between them. Figure~\ref{fig:crown_jewel} shows an example on a graph for recommendations. Given a GNN and a predicted link between user $u_1$ and item $i_1$, human-interpretable explanations may be based on the user's preference of attributes (e.g., user $u_1$ bought item $i2$ that shared the same attribute $a_1$ as item $i_1$) or collaborative filtering (e.g, user $u_1$ had a similar preference as user $u_2$ because they both bought item $i_3$ and user $u_2$ bought item $i_1$, so that user $u_1$ would like item $i_1$). Both explanations boil down to paths. Secondly, paths have a considerably smaller search space than general subgraphs. As we will see in Proposition~\ref{prop:num_paths}, compared to the expected number of edge-induced subgraphs, the expected number of paths grows strictly slower and becomes negligible. Therefore, path explanations exclude many less-meaningful subgraph candidates, making the explanation generation much more straightforward and accurate.

To this end, we propose \underline{Pa}th-based \underline{G}NN \underline{E}xplanation for heterogeneous \underline{Link} prediction (\method), which achieves a better explanation AUC and scales linearly in the number of edges (see Figure~\ref{fig:crown_jewel2}). 
We first perform \textit{k-core pruning}~\cite{kcore-def} to help find paths and improve scalability. Then we do \textit{heterogeneous path-enforcing} mask learning to determine important paths, which handles heterogeneity and enforces the explanation edges to form paths connecting source to target. 
In summary, the contributions of our method are:

\begin{itemize}[leftmargin=10pt]
    \item {\bf \interpretability:} \method produces more interpretable explanations in path forms and quantitatively improves explanation AUC over baselines.
    \item {\bf \scalability}: \method reduces the explanation search space by magnitudes from subgraph finding to path finding and scales linearly in the number of graph edges.
    \item {\bf \heterogeneity:} \method works on heterogeneous graphs and leverages edge-type information to generate better explanations.
\end{itemize}

\begin{figure}[t]
\centering
\begin{subfigure}[b]{0.53\columnwidth}
\centering
\includegraphics[width=0.99\textwidth]{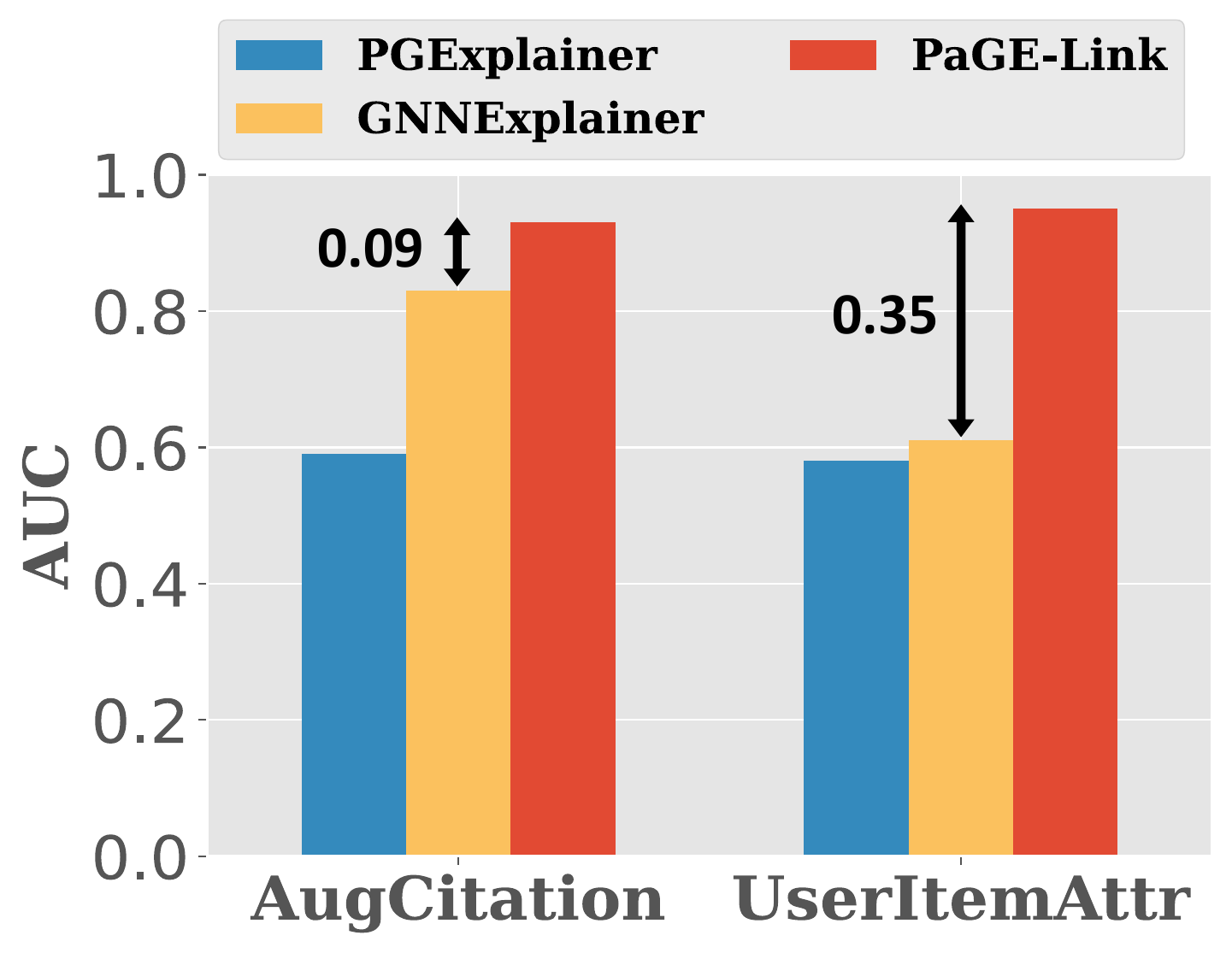}
\centering
\end{subfigure}
\hfill
\begin{subfigure}[b]{0.46\columnwidth}
\centering
\includegraphics[width=0.99\columnwidth]{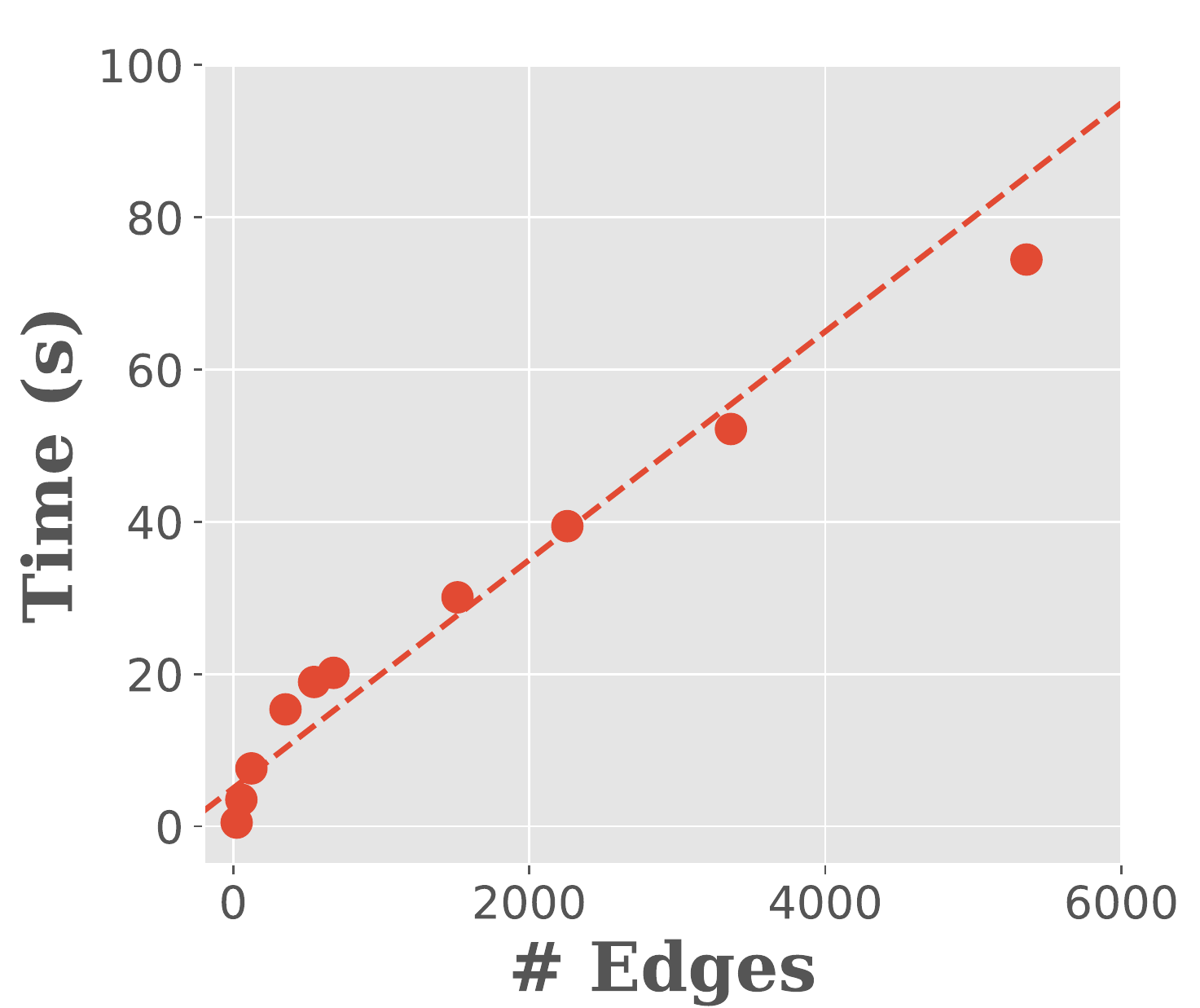}
\end{subfigure}
\caption{(a) \method outperforms GNNExplainer and PGExplainer in terms of explanation AUC on the citation graph and the user-item graph. (b) The running time of \method scales linearly in the number of graph edges.}
\label{fig:crown_jewel2}
\Description{PaGE-Link shows higher AUC scores than baselines, and it scales linearly in the number of edges in the graph.}
\end{figure}

\section{Related work}\label{sec:related}
We review relevant research on (a) GNNs (b) GNN explanation (c) recommendation explanation and (d) paths for LP. We summarize the properties of \method vs. representative methods in Table \ref{tab:salesman}.

\paragraph{GNNs}
GNNs are a family of ML models on graphs~\cite{gcn, gat, gin}. They take graph structure and node/edge features as input and output node representations by transforming and aggregating features of nodes’ (multi-hop) neighbors. The node representations can be used for LP and achieved great results on LP applications~\cite{zhang2020revisiting, zhang2018link, zhang2019star, mao2021ultragcn, wu2020graph, zhao2022learning, guo2022linkless}. We review GNN-based LP models in Section~\ref{sec:preliminary}.

\paragraph{GNN explanation} GNN explanation was studied for node and graph classification, where the explanation is defined as an important subgraph. Existing methods majorly differ in their definition of importance and subgraph selection methods. GNNExplainer~\cite{gnnexplainer} selects edge-induced subgraphs by learning fully parameterized masks on graph edges and node features, where the mutual information (MI) between the masked graph and the prediction made with the original graph is maximized. PGExplainer~\cite{pgexplainer} adopts the same MI importance but trains a mask predictor to generate a discrete mask instead. Other popular importance measures are game theory values. SubgraphX~\cite{subgraphx} uses the Shapley value~\cite{shapley} and performs Monte Carlo Tree Search (MCTS) on subgraphs. GStarX~\cite{zhang2022gstarx} uses a structure-aware HN value~\cite{hn_value} to measure the importance of nodes and generates the important-node-induced subgraph. There are more studies from other perspectives that are less related to this work, i.e., surrogate models~\cite{graphlime, pgm_explainer}, counterfactual explanations~\cite{lucic2022cf}, and causality~\cite{gem, orphicx}, for which \cite{taxonomy} provides a good review. While these methods produce subgraphs as explanations, what makes a good explanation is a complex topic, especially how to meet ``stakeholders’ desiderata''~\cite{langer2021we}. Our work differs from all above since we focus on a new task of explaining heterogeneous LP, and we generate paths instead of unrestricted subgraphs as explanations. The interpretability of paths makes our method advantaged especially when stakeholders have less ML background. 

\paragraph{Recommendation explanation} This line of works explains why a recommendation is made \cite{rec_explain_survey}. J-RECS~\cite{jrecs} generates recommendation explanations on product graphs using a justification score that balances item relevance and diversity. PRINCE~\cite{prince_2020} produces end-user explanations as a set of minimal actions performed by the user on graphs with users, items, reviews,
and categories. The set of actions is selected using counterfactual evidence.
Typically, recommendations on graphs can be formalized as an LP task. However, the recommendation explanation problem differs from explaining GNNs for LP because the recommendation data may not be graphs, and the models to be explained are primarily not GNN-based~\cite{wang2019explainable}. GNNs have their unique message passing procedure, and GNN-based LP corresponds to more general applications beyond recommendation, e.g., drug repurposing~\cite{ioannidis2020few}, and knowledge graph completion~\cite{nickel2015review, cheng-etal-2021-uniker}.
Thus, recommendation explanation is related to but not directly comparable to GNN explanation.

\paragraph{Paths} Paths are important in graph ML, and many LP methods are path-based, such as graph distance~\cite{liben2007link}, Katz index~\cite{katz}, SimRank~\cite{jeh2002simrank}, and PathSim~\cite{sun2011pathsim}. Paths have also been used to capture the relationship between a pair of nodes. For example, the ``connection subgraphs''~\cite{faloutsos2004fast} find paths between the source and the target based on electricity analogs. In general, although black-box GNNs recently outperform path-based methods in LP accuracy, we embrace paths for their interpretability for LP explanation.

\def\boldqm{\textbf{?}}
\def\rot{\rotatebox{72}}
\begin{table}[t]
\center
\footnotesize
\caption{Methods and desired explanation properties.
A question mark (\boldqm{}) means ``unclear'', or ``maybe, after
non-trivial extensions''. "Rec. Exp." stands for the general recommendation explanation methods.}
\begin{tabular}{@{}lcccccc|c@{}}
\toprule
\multicolumn{1}{l}{Methods} & \rot{GNNExp~\scriptsize{\cite{gnnexplainer}}} & \rot{PGExp~\scriptsize{\cite{pgexplainer}}} & \rot{SubgraphX~\scriptsize{\cite{subgraphx}}} & \rot{J-RECS~\scriptsize{\cite{jrecs}}} & \rot{PRINCE~\scriptsize{\cite{prince_2020}}} &  \rot{Rec. Exp.~\scriptsize{\cite{rec_explain_survey}}} & \rot{\method} \\ \midrule
On Graphs & \checkmark & \checkmark & \checkmark  & \checkmark & \checkmark & \boldqm & \checkmark \\
Explains GNN & \checkmark& \checkmark & \checkmark & & & & \checkmark \\
Explains LP & \boldqm & \boldqm & \boldqm & \checkmark & \checkmark &  \checkmark & \checkmark \\
Connection &  & & & \boldqm & \boldqm & \boldqm & \checkmark \\ 
Scalability & \checkmark & \checkmark & & \checkmark & \boldqm & \boldqm & \checkmark \\
Heterogeneity & & & \checkmark & \checkmark & \checkmark & \boldqm & \checkmark \\
\bottomrule
\end{tabular}
\label{tab:salesman}
\Description{PaGE-Link wins on desired properties of good explanations.}
\end{table}

\section{Notations and preliminary}\label{sec:preliminary}
In this section, we define necessary notations, summarize them in Table \ref{tab:notation}, and review the GNN-based LP models.

\begin{definition} \label{def:hetero}
A heterogeneous graph is defined as a directed graph $\gG = (\V, \E)$ associated with a node type
mapping function $\phi: \V \rightarrow \A$ and an edge type mapping function
$\tau: \E \rightarrow \R$. Each node $v \in \V$ belongs to one node type $\phi(v) \in \A$ and each edge $e \in \E$ belongs to one edge type $\tau(e) \in \R$.    
\end{definition}

Let $\Phi(\cdot, \cdot)$ denote a trained GNN-based model for predicting the missing links in $\gG$, where a prediction $Y = \Phi(\gG, (s, t))$ denotes the predicted link between a source node $s$ and a target node $t$. The model $\Phi$ learns a conditional distribution $P_{\Phi}(Y | \gG, (s, t))$ of the binary random variable $Y$. The commonly used GNN-based LP models~\cite{zhang2018link, zhu2021neural, zhao2022learning} involve two steps. The first step is to generate node representations $(\vh_s, \vh_t)$ of $(s, t)$ with an $L$-hop GNN encoder. The second step is to apply a prediction head on $(\vh_s, \vh_t)$ to get the prediction of $Y$. An example prediction head is an inner product.

To explain $\Phi(\gG, (s, t))$ with an $L$-Layer GNN encoder, we restrict to the \textit{computation graph} $\gG_c = (\V_c, \E_c)$. $\gG_c$ is the $L$-hop ego-graph of the predicted pair $(s, t)$, i.e., the subgraph with node set $\V_c = \{v \in V | dist(v, s) \leq L \text{ or } dist(v, t) \leq L\}$. It is called a computation graph because the $L$-layer GNN only collects messages from the $L$-hop neighbors of $s$ and $t$ to compute $\vh_s$ and $\vh_t$. The LP result is thus fully determined by $\gG_c$, i.e., $\Phi(\gG, (s, t)) \equiv \Phi(\gG_c, (s, t))$. Figure~\ref{fig:p-b} shows a 2-hop ego-graph of $u_1$ and $i_1$, where $u_3$ and $a^1_3$ are excluded since they are more than 2 hops away from either $u_1$ or $i_1$.

\begin{table}[t]
\center
\small
\caption{Notation table}
\begin{tabular}{@{}l|l@{}}
\toprule
\multicolumn{1}{l}{Notation} & Definition and description \\ \midrule
$\gG = (\V, \E)$ & a heterogeneous graph $\gG$, node set $\V$, and edge set $\E$ \\
$\phi: \V \rightarrow \A$ &  a node type mapping function \\
$\tau: \E \rightarrow \R$ &  an edge type mapping function \\
$D_v$ & the degree of node $v \in \V$ \\
$\E^r $ & edges with type $r \in \R$, i.e., $\E^r = \{e \in \E | \tau(e) = r\}$ \\
$\Phi(\cdot, \cdot)$ & the GNN-based LP model to explain \\
$(s, t)$ & the source and target node for the predicted link \\
$\vh_s$ \& $\vh_t$ & the node representations for $s$ \& $t$ \\
$Y = \Phi(\gG, (s, t))$ & the link prediction of the node pair $(s, t)$ \\
$\gG_c = (\V_c, \E_c)$ & the computation graph, i.e., L-hop ego-graph of $(s, t)$ \\
\bottomrule
\end{tabular}
\label{tab:notation}
\Description{Notations.}
\end{table}

\section{Proposed problem formulation: link-prediction explanation }\label{sec:problem}
\begin{figure*}[t]
\centering
\begin{subfigure}[t]{0.22\textwidth}
\centering
\includegraphics[width=0.9\textwidth]{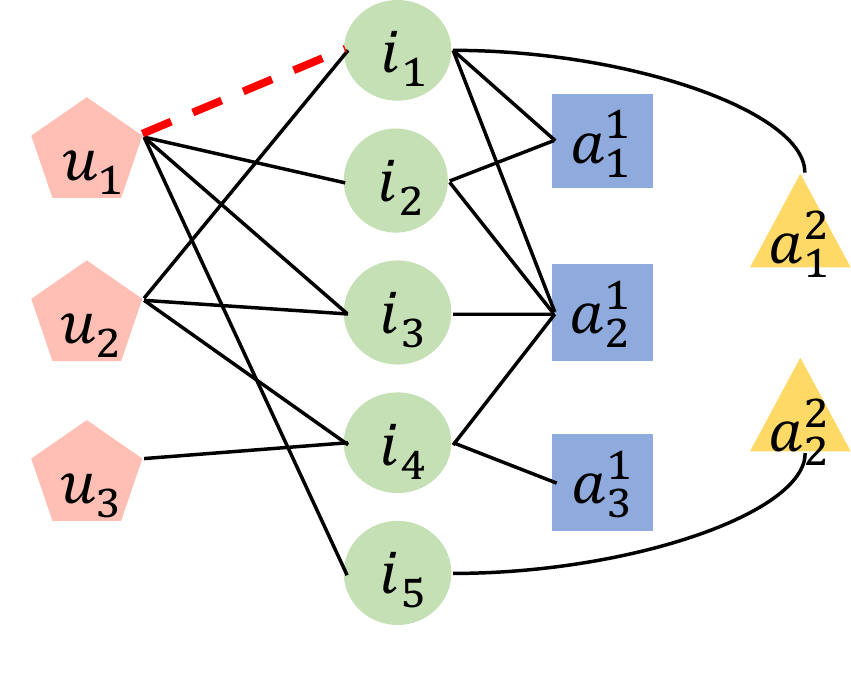}
\caption{A GNN predicted link $(u_1, i_1)$ ({\color{red} dashed red}) that needs explanation.}
\label{fig:p-a}
\end{subfigure}
\hfill
\begin{subfigure}[t]{0.41\textwidth}
\centering
\includegraphics[width=0.9\textwidth]{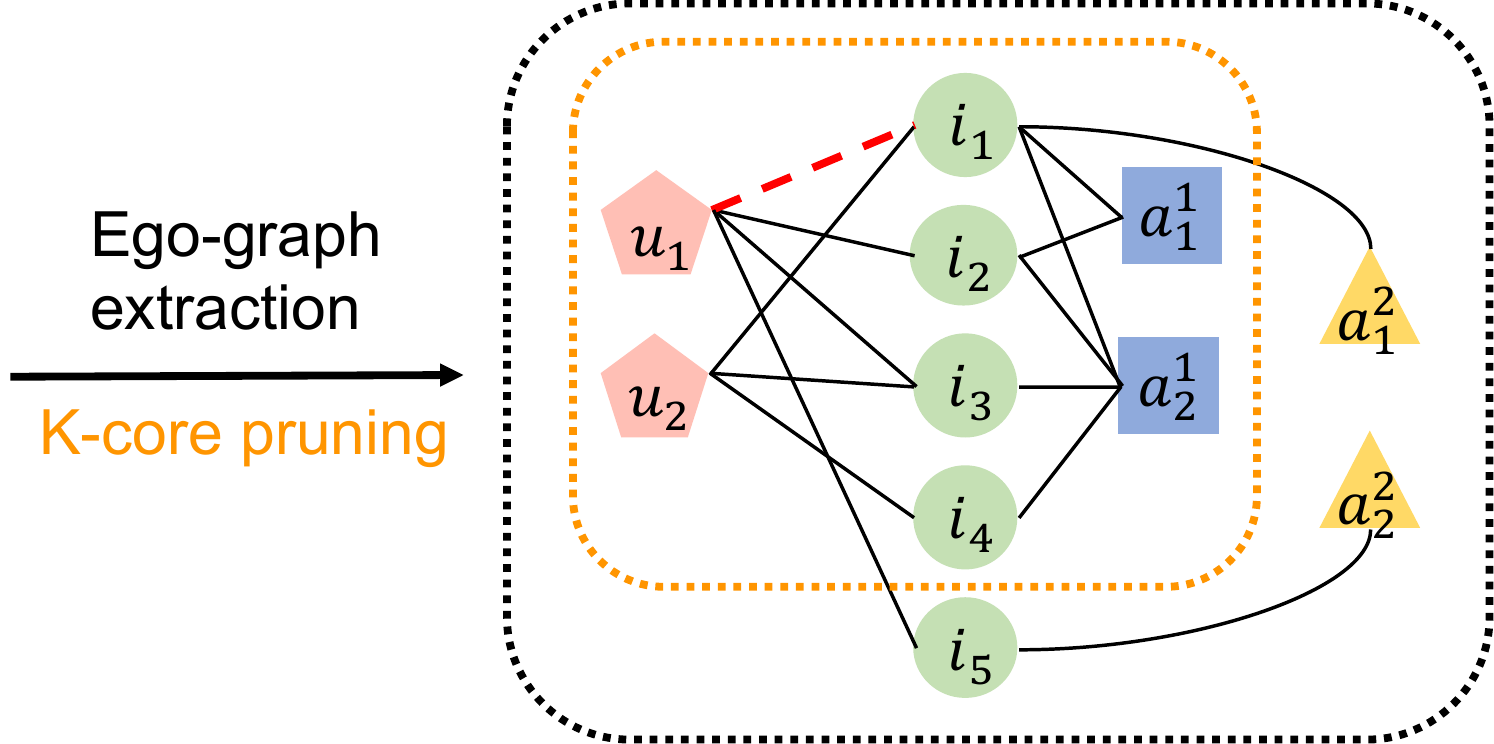}
\caption{Extract 2-hop ego-graph of $(u_1, i_1)$ excluding $u_3$ and $a^1_3$ (black box). Then prune it to get the k-core excluding $i_5$, $a^2_1$, and $a^2_2$ ({\color{YellowOrange} orange box}).}
\label{fig:p-b}
\end{subfigure}
\hfill
\begin{subfigure}[t]{0.33\textwidth}
\centering
\includegraphics[width=0.9\textwidth]{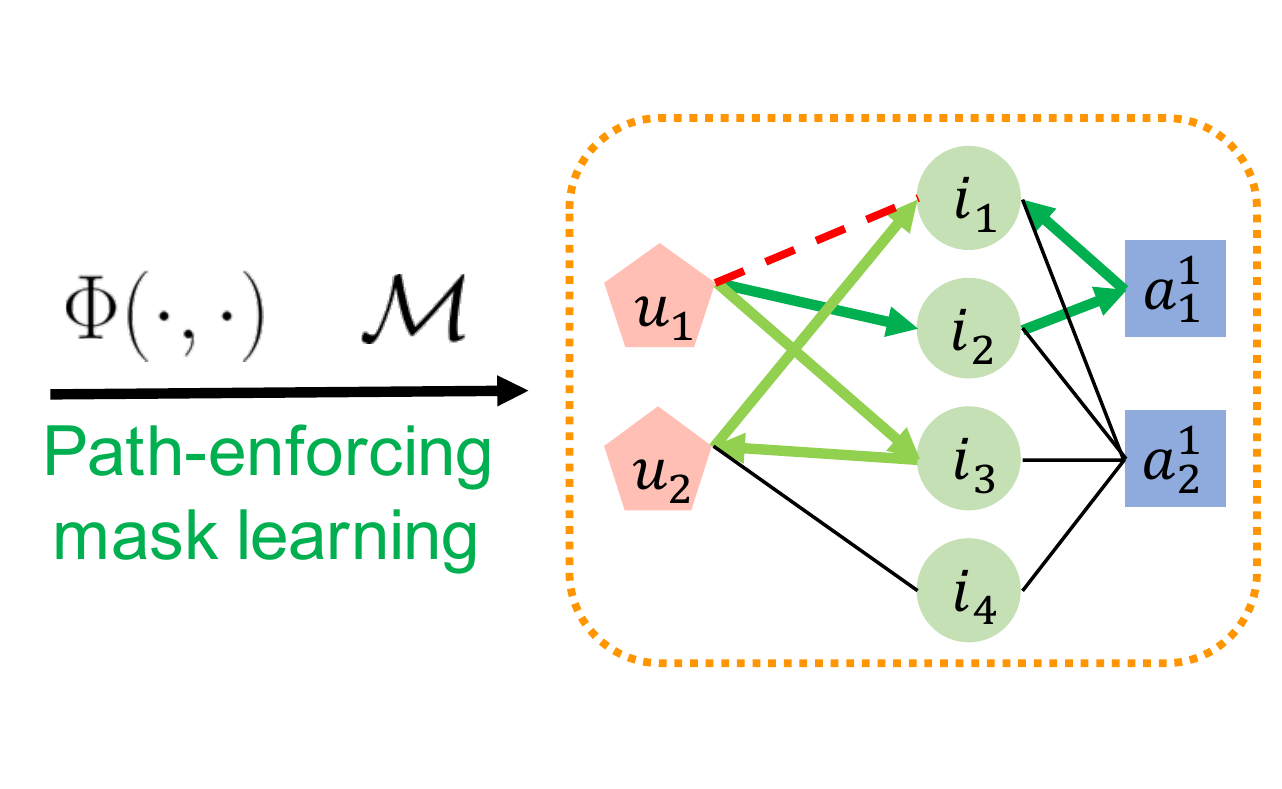}
\caption{Human-interpretable path explanations $(u_1, i_2, a_1^1, i_1)$ and $(u_1, i_3, u_2, i_1)$ ({\color{ForestGreen} green arrows}) that capture the connection between $u_1$ and $i_1$.} 
\label{fig:p-c}
\end{subfigure}
\caption{\method on a graph with user nodes $u$, item nodes $i$, and two attribute types $a^1$ and $a^2$. (Best viewed in color.) }
\label{fig:pipeline}
\Description{Framework of PaGE-Link, including the k-core pruning module and the path-enforcing mask learning module.}
\end{figure*}

In this work, we address a \textit{post hoc} and \textit{instance-level} GNN explanation problem. The post hoc means the model $\Phi(\cdot, \cdot)$ has been trained. To generate explanations, we won't change its architecture or parameters. The instance level means we generate an explanation for the prediction of each instance $(s, t)$. Specifically, the explanation method answers the question of why a missing link is predicted by $\Phi(\cdot, \cdot)$. In a practical web recommendation system, this question can be ``\textit{why an item is recommended to a user by the model}''.

An explanation for a GNN prediction should be some substructure in $\gGc$, and it should also be concise, i.e., limited by a size budget $B$. This is because an explanation with a large size is often neither informative nor interpretable, for example, an extreme case is that $\gGc$ could be a non-informative explanation for itself. Also, a fair comparison between different explanations should consume the same budget. In the following, we define budget $B$ as the maximum number of edges included in the explanation.

We list three desirable properties for a GNN explanation method on heterogeneous LP: capturing the connection between the source node and the target node, scalable to large graphs, and addressing graph heterogeneity. Using a path-based method inherently possesses all the properties. Paths capture the connection between a pair of nodes and can be transferred to human-interpretable explanations. Besides, the search space of paths with the fixed source node and the target node is greatly reduced compared to edge-induced subgraphs. Given the ego-graph $\gG_c$ of $s$ and $t$, the number of paths between $s$ and $t$ and the number of edge-induced subgraphs in $\gG_c$ both rely on the structure of $\gGc$. However, they can be estimated using random graph approximations. The next proposition on random graphs shows that the expected number of paths grows strictly slower than the expected number of edge-induced subgraphs as the random graph grows. Also, the expected number of paths becomes insignificant for large graphs. 
\begin{proposition} \label{prop:num_paths}
Let $\gG(n, d)$ be a random graph with n nodes and density d, i.e., there are $m=d \binom{n}{2}$ edges chosen uniformly randomly from all node pairs. Let $Z_{n,d}$ be the expected number of paths between any pair of nodes. Let $S_{n, d}$ be the expected number of edge-induced subgraphs. Then $Z_{n,d} = o(S_{n,d})$, i.e., $\lim_{n \to \infty} \frac{Z_{n,d}}{S_{n,d}} = 0$.
\end{proposition}
\begin{proof}
    In Appendix \ref{app:num_paths_proof}.
\end{proof}

Paths are also a natural choice for LP explanations on heterogeneous graphs. On homogeneous graphs, features are important for prediction and explanation. A $s$-$t$ link may be predicted because of the feature similarity of node $s$ and node $t$. However, the heterogeneous graphs we focus on, as defined in Definition~\ref{def:hetero}, often do not store feature information but explicitly model it using new node and edge types. For example, for the heterogeneous graph in Figure~\ref{fig:p-a}, instead of making it a user-item graph and assigning each item node a two-dimensional feature with attributes $a^1$ and $a^2$, the attribute nodes are explicitly created and connected to the item nodes. Then an explanation like ``$i_1$ and $i_2$ share node feature $a^1_1$'' on a homogeneous graph is transferred to ``$i_1$ and $i_2$ are connected through the attribute node $a^1_1$'' on a heterogeneous graph.

Given the advantages of paths over general subgraphs on connection interpretability, scalability, and their capability to capture feature similarity on heterogeneous graphs, we use paths to explain GNNs for heterogeneous LP. Our design principle is that a good explanation should be concise and informative, so we define the explanation to contain only \textit{short} paths \textit{without high-degree} nodes. Long paths are less desirable since they could correspond to unnecessarily complicated connections, making the explanation neither concise nor convincing. For example, in Figure~\ref{fig:p-c}, the long path $(u_1, i_3, a^1_2, i_2, a^1_1, i_1)$ is not ideal since it takes four hops to go from item $i_3$ to the item $i_1$, making it less persuasive to be interpreted as ``item1 and item3 are similar so item1 should be recommended''. Paths containing high-degree nodes are also less desirable because high-degree nodes are often generic, and a path going through them is not as informative. In the same figure, all paths containing node $a^1_2$ are less desirable because $a^1_2$ has a high degree and connects to all the items in the graph. A real example of a generic attribute is the attribute ``grocery'' connecting to both ``vanilla ice cream'' and ``vanilla cookie''. When ``vanilla ice cream'' is recommended to a person who bought ``vanilla cookie'', explaining this recommendation with a path going through ``grocery'' is not very informative since ``grocery'' connects many items. In contrast, a good informative path explanation should go through the attribute ``vanilla'', which only connects to vanilla-flavored items and has a much lower degree. 

We formalize the GNN explanation for heterogeneous LP as:

\begin{problem}[]
\label{prob:path}
Generating path-based explanations for a predicted link between node $s$ and $t$:
\begin{itemize}[leftmargin=10pt]
    \item \textbf{Given}
    \begin{itemize}[leftmargin=5pt]
        \item a trained GNN-based LP model $\Phi(\cdot, \cdot)$,
        \item a heterogeneous computation graph $\gGc$ of $s$ and $t$,
        \item a budget $B$ of the maximum number of edges in the explanation,
    \end{itemize}
    \item \textbf{Find} an explanation $\P = $ \{ $p | p$ is a $s$-$t$ path with maximum length $l_{max}$ and degree of each node less than $D_{max}$ \}, $|\P|l_{max} \leq B$, 
    \item \textbf{By optimizing} $p \in \P$ to be influential to the prediction, concise, and informative.
\end{itemize}
\end{problem}

\section{Proposed method: PaGE-Link}\label{sec:method}
This section details \method. \method has two modules: (i) a $k$-core pruning module to eliminate spurious neighbors and improve speed, and (ii) a heterogeneous path-enforcing mask learning module to identify important paths. An illustration is in Figure~\ref{fig:pipeline}.

\subsection{The k-core Pruning}
The \textit{$k$-core pruning} module of \method reduces the complexity of $\gG_c$. The $k$-core of a graph is defined as the unique maximal subgraph with a minimum node degree $k$ \cite{kcore-def}. We use the superscript $k$ to denote the $k$-core, i.e., $\gG_c^k = (\E_c^k, \V_c^k)$ for the $k$-core of $\gG_c$. The $k$-core pruning is a recursive algorithm that removes nodes $v \in \V$ such that their degrees $D_v < k$, until the remaining subgraph only has nodes with $D_v \geq k$, which gives the $k$-core. The difference in nodes between a $(k+1)$-core and a $k$-core is called the $k$-shell. The nodes in the orange box of Figure~\ref{fig:p-b} is an example of a $2$-core pruned from the $2$-hop ego-graph, where node $a^2_1$ and $a^2_2$ are pruned in the first iteration because they are degree one. Node $i_5$ is recursively pruned because it becomes degree one after node $a^2_2$ is pruned. All those three nodes belong to the $1$-shell. We perform $k$-core pruning to help path finding because the pruned $k$-shell nodes are unlikely to be part of meaningful paths when $k$ is small. For example, the $1$-shell nodes are either leaf nodes or will become leaf nodes during the recursive pruning, which will never be part of a path unless $s$ or $t$ are one of these $1$-shell nodes. The $k$-core pruning module in \method is modified from the standard $k$-core pruning by adding a condition of never pruning $s$ and $t$. 

The following theorem shows that for a random graph $\gG(n, d)$, $k$-core will reduce the expected number of nodes by a factor of $\delta_{\V}(n, d, k)$ and reduce the expected number of edges by a factor of $\delta_{\E}(n, d, k)$. Both factors are functions of $n$, $d$, and $k$. We defer the exact expressions of these two factors 
in Appendix \ref{app:kcore_size}, since they are only implicitly defined based on Poisson distribution. Numerically, for a random $\gG(n, d)$ with average node degree $d(n-1) = 7$, its 5-core has $\delta_{\V}(n, d, 5)$ and $\delta_{\E}(n, d, 5)$ both $ \approx 0.69$. 

\begin{theorem}[Pittel, Spencer and Wormald \cite{kcore-emergence}] \label{thm:kcore_size}
Let $\gG(n, d)$ be a random graph with $m$ edges as in Proposition \ref{prop:num_paths}. Let $\gG^k(n, d) = (\V^k(n, d), \E^k(n, d))$ be the nonempty $k$-core of $\gG(n, d)$. Then $\gG^k(n, d)$ contain $\delta_{\V}(n, d, k) n$ nodes and $\delta_{\E}(n, d, k) m$ edges with high probability for large n, i.e., $|\V^k(n, d)| / n \xrightarrow{p} \delta_{\V}(n, d, k)$ and $|\E^k(n, d)| / m \xrightarrow{p} \delta_{\E}(n, d, k)$ ($\xrightarrow{p}$ stands for convergence in probability). \end{theorem}

\begin{proof}
    Please refer to Appendix \ref{app:kcore_size} and \cite{kcore-emergence}.
\end{proof}

The $k$-core pruning helps reduce the graph complexity and accelerates path finding. One concern is whether it prunes too much and disconnects $s$ and $t$. We found that such a situation is very unlikely to happen in practice. To be specific, we focus on explaining positively predicted links, e.g. why an item is recommended to a user by the model. Negative predictions, e.g., why an arbitrary item is not recommended to a user by the model, are less useful in practice and thus not in the scope of our explanation. $(s, t)$ node pairs are usually connected by many paths in a practical $\gG$~\cite{smallworld}, and positive link predictions are rarely made between disconnected or weakly-connected $(s, t)$. Empirically, we observe that there are usually too many paths connecting a positively predicted $(s, t)$ instead of no paths, even in the $k$-core. Therefore, an optional step to enhance pruning is to remove nodes with super-high degrees. As we discussed in Section \ref{sec:problem}, high-degree nodes are often generic and less informative. Removing them can be a complement to k-core to further reduce complexity and improve path quality.

\subsection{Heterogeneous Path-Enforcing Mask Learning}
The second module of \method learns heterogeneous masks to find important path-forming edges. We perform mask learning to select edges from the $k$-core-pruned computation graph. For notation simplicity in this section, we use $\gG = (\V, \E)$ to denote the graph for mask learning to save superscripts and subscripts, and $\gG_c^k$ is the actual graph in the complete version of our algorithm. 

The idea is to learn a mask over all edges of all edge types to select the important edges. Let $\E^r = \{e \in \E | \tau(e) = r\}$ be edges with type $r \in \R$. Let $\M = \{\M^r\}_{r=1}^{|\R|}$ be learnable masks of all edge types, with $\M^r \in \mathbb{R}^{|\E^r|}$ corresponds type $r$. We denote applying $\M^r$ on its corresponding edge type by $\E^r \odot \sigma(\M^r)$, where $\sigma$ is the sigmoid function, and $\odot$ is the element-wise product. Similarly, we also overload the notation $\odot$ to indicate applying the set of masks on all types of edges, i.e., $\E \odot \sigma(\M) = \cup_{r \in \R} \{\E^r \odot \sigma(\M^r) \}$. We call the graph with the edge set $\E \odot \sigma(\M)$ a \textit{masked graph}. Applying a mask on graph edges will change the edge weights, which makes GNNs pass more information between nodes connected by highly-weighted edges and less on others. The general idea of mask learning is to learn an $\M$ that produces high weights for important edges and low weights for others. To learn an $\M$ that better fits the LP explanation, we measure edge importance from two perspectives: important edges should be both influential for the model prediction and form meaningful paths. Below, we introduce two loss terms $\Ls_{pred}$ and $\Ls_{path}$ for achieving these two measurements.

$\Ls_{pred}$ is to learn to select influential edges for model prediction. The idea is to do a perturbation-based explanation, where parts of the input are considered important if perturbing them changes the model prediction significantly. In the graph sense, if removing an edge $e$ significantly influences the prediction, then $e$ is a critical counterfactual edge that should be part of the explanation. This idea can be formalized as maximizing the mutual information between the masked graph and the original graph prediction $Y$, which is equivalent to minimizing the prediction loss
\begin{equation} \label{eq:pred_loss}
    \Ls_{pred}(\M) = - \log P_{\Phi} (Y = 1 | \gG = (\V, \E \odot \sigma(\M)), (s, t)).
\end{equation}
 
$\Ls_{pred}(\M)$ has a straightforward meaning, which says the masked subgraph should provide enough information for predicting the missing link $(s, t)$ as the whole graph. Since the original prediction is a constant, $\Ls_{pred}(\M)$ can also be interpreted as the performance drop after the mask is applied to the graph. A well-masked graph should give a minimum performance drop. Regularizations of the mask entropy and mask norm are often included in $\Ls_{pred}(\M)$ to encourage the mask to be discrete and sparse. 

$\Ls_{path}$ is the loss term for $\M$ to learn to select path-forming edges. The idea is to first identify a set of candidate edges denoted by $\E_{path}$ (specified below), where these edges can form concise and informative paths, and then optimize $\Ls_{path} (\M)$ to enforce the mask weights for $e \in \E_{path}$ to increase and mask weights for $e \notin \E_{path}$ to decrease. We considered a weighted average of these two forces balanced by hyperparameters $\alpha$ and $\beta$,
\begin{equation} \label{eq:path_loss}
    \Ls_{path}(\M) = - \sum_{r \in \R} (\alpha \sum_{\substack{e \in \E_{path} \\ \tau(e) = r}} \M_{e}^r - \beta \sum_{\substack{e \in \E, e \notin \E_{path} \\ \tau(e) = r}}  \M_{e}^r).
\end{equation}

The key question for computing $\Ls_{path}(\M)$ is to find a good $\E_{path}$ containing edges of concise and informative paths. As in Section \ref{sec:problem}, paths with these two desired properties should be short and without high-degree generic nodes. We thus define a score function of a path $p$ reflecting these two properties as below 
\begin{align}
    Score(p) &= \log \prod_{\substack{e \in p \\ e = (u, v)}} \frac{P(e)}{D_v} = \sum_{\substack{e \in p \\ e = (u, v)}} Score(e), \\
    Score(e) &= \log \sigma (\M^{\tau(e)}_{e}) - \log(D_v). \label{eq:path_score}
\end{align} 
In this score function, $\M$ gives the probability of $e$ to be included in the explanation, i.e., $P(e) = \sigma(\M_{e}^{\tau(e)})$. To get the importance of a path, we first use a mean-field approximation for the joint probability by multiplying $P(e)$ together, and we normalize each $P(e)$ for edge $e = (u, v)$ by its target node degree $D_v$. Then, we perform log transformation, which improves numerical stability for multiplying many edges with small $P(e)$ or large $D_v$ and break down a path score to a summation of edge scores $Score(e)$ that are easier to work with. This path score function captures both desired properties mentioned above. A path score will be high if the edges on it have high probabilities and these edges are linked to nodes with low degrees. Finding paths with the highest $Score(p)$ can be implemented using Dijkstra's shortest path algorithm \cite{dijkstra1959note}, where the distance represented by each edge is set to be the negative score of the edge, i.e., $- Score(e)$. We let $\E_{path}$ be the set of edges in the top five shortest paths found by Dijkstra's algorithm.

\subsection{Mask Optimization and Path Generation}
We optimize $\M$ with both $\Ls_{pred}$ and $\Ls_{path}$. $\Ls_{pred}$ will increase the weights of the prediction-influential edges. $\Ls_{path}$ will further increase the weights of the path-forming edges that are also highly weighted by the current $\M$ and decrease other weights. Finally, after the mask learning converges, we run one more shortest-path algorithm to generate paths from the final $\M$ and select the top paths according to budget $B$ to get the explanation $\P$ defined in Section \ref{sec:problem}. A pseudo-code of \method is shown in Algorithm \ref{algo:code}.

\begin{algorithm}[t]
  \caption{\method}
  \label{alg:pagelink}
\begin{algorithmic}
  \State {\bfseries Input:} heterogeneous graph $\gG$,  trained GNN-based LP model $\Phi(\cdot, \cdot)$, predicted link $(s, t)$, size budget $B$, k for k-core, hyperparameters $\alpha$ and $\beta$, learning rate $\eta$, maximum iterations $T$.
\State {\bfseries Output:} Explanation as a set of paths $\P$.
\State Extract the computation graph $\gG_c$;
\State Prune $\gG_c$ for the k-core $\gG_c^k$;
\State Initialize $\M^{(0)}$;
\State $t = 0$;
\While{ $\M^{(t)}$ not converge and $t < T$}
\State Compute $\Ls_{pred}(\M^{(t)})$; \Comment{ Eq.(\ref{eq:pred_loss})}
\State Compute $Score(e)$ for each edge $e$; \Comment{ Eq.(\ref{eq:path_score})}
\State Construct $\E_{path}$ by finding shortest paths on $\gG_c^k$ with edge distance $-Score(e)$;
\State Compute $\Ls_{path}(\M^{(t)})$ according to $\E_{path}$; \Comment{ Eq.(\ref{eq:path_loss})}
\State $\M^{(t+1)} = \M^{(t)} - \eta \nabla (\Ls_{pred}(\M^{(t)})$ + $\Ls_{path}(\M^{(t)}))$; 
\State t += 1;
\EndWhile
\State $\P = $ Under budget $B$, the top shortest paths on $\gG_c^k$ with edge distance $-Score(e)$;
\State \textbf{Return:} $\P$.
\end{algorithmic}
\label{algo:code}
\end{algorithm}

\subsection{Complexity Analysis} \label{subsec:complexity}
\begin{table}[t]
\center
\footnotesize
\caption{Time complexity of \method and other methods. }
\begin{tabular}{@{}ccc|c@{}}
\toprule
GNNExp~\cite{gnnexplainer} & PGExp~\cite{pgexplainer} & SubgraphX~\cite{subgraphx} & \method (ours)  \\ \midrule    
$O(|\E_c| T)$ & $O(|\E| T)$ \big/ $O(|\E_c|)$ & $\Theta(|\V_c|\hat D^{2B_{node} - 2})$ & $O(|\E_c| + |\E_c^k|T)$\\
\bottomrule
\end{tabular}
\label{tab:complexity}
\Description{PaGE-Link has better time complexity than baselines}
\end{table}

In Table \ref{tab:complexity}, we summarize the time complexity of \method and representative existing methods for explaining a prediction with computation graph $\gG_c = (\V_c, \E_c)$ on a full graph $\gG = (\V, \E)$. Let $T$ be the mask learning epochs. GNNExplainer has complexity $|\E_c|T$ as it learns a mask on $\E_c$. PGExplainer has a training stage and an inference stage (separated by / in the table). The inference stage is linear in $|\E_c|$, but the training stage covers edges in the entire graph and thus scales in $O(|\E| T)$. SubgraphX has a much higher time complexity exponential in $|\V_c|$, so a size budget of $B_{node}$ nodes is forced to replace $|\V_c|$, and $\hat D = \max_{v \in \V} D_v$ denotes the maximum degree (derivation in Appendix \ref{app:complexity}). For \method, the k-core pruning step is linear in $|\E_c|$. The mask learning with Dijkstra's algorithm has complexity $|\E_c^k|T$. \method has a better complexity than existing methods since $|\E_c^k|$ is usually smaller than $|\E_c|$ (see Theorem \ref{thm:kcore_size}), and \method often converges faster, i.e., has a smaller $T$, as the space of candidate explanations is smaller (see Proposition \ref{prop:num_paths}) and noisy nodes are pruned.

\section{Experiments} \label{sec:experiment}
In this section, we conduct empirical studies to evaluate explanations generated by \method. Evaluation is a general challenge when studying model explainability since standard datasets do not have ground truth explanations. Many works~\cite{gnnexplainer, pgexplainer} use synthetic benchmarks, but no benchmarks are available for evaluating GNN explanations for heterogeneous LP. Therefore, we generate an augmented graph and a synthetic graph to evaluate explanations. They allow us to generate ground truth explanation patterns and evaluate explainers quantitatively.

\subsection{Datasets} \label{subsec:dataset}

\begin{figure*}[t]
\centering
\begin{subfigure}[t]{0.25\textwidth}
\includegraphics[width=0.9\textwidth]{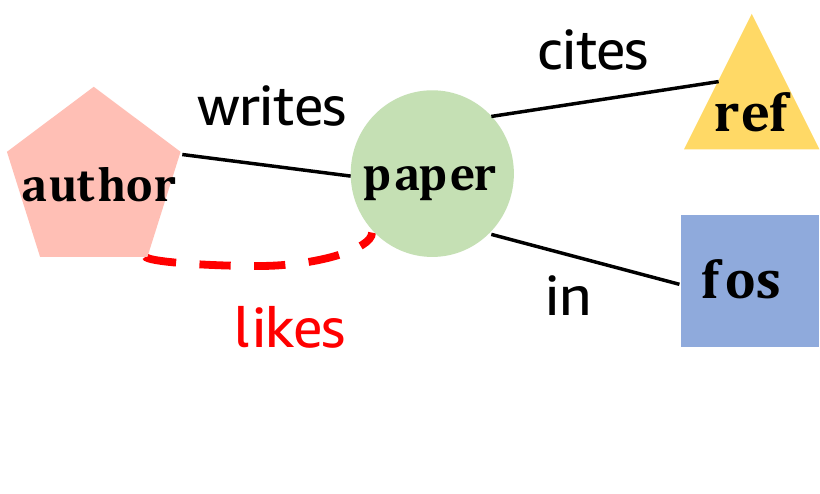}
\caption{Schema of \aminer. ``writes'', ``cites'', and ``in'' edges are original. The ``likes'' edges ({\color{red} dashed red}) are augmented for prediction.} 
\centering
\label{fig:create_graph_a}
\end{subfigure}
\hfill
\begin{subfigure}[t]{0.73\textwidth}
\centering
\includegraphics[width=0.9\textwidth]{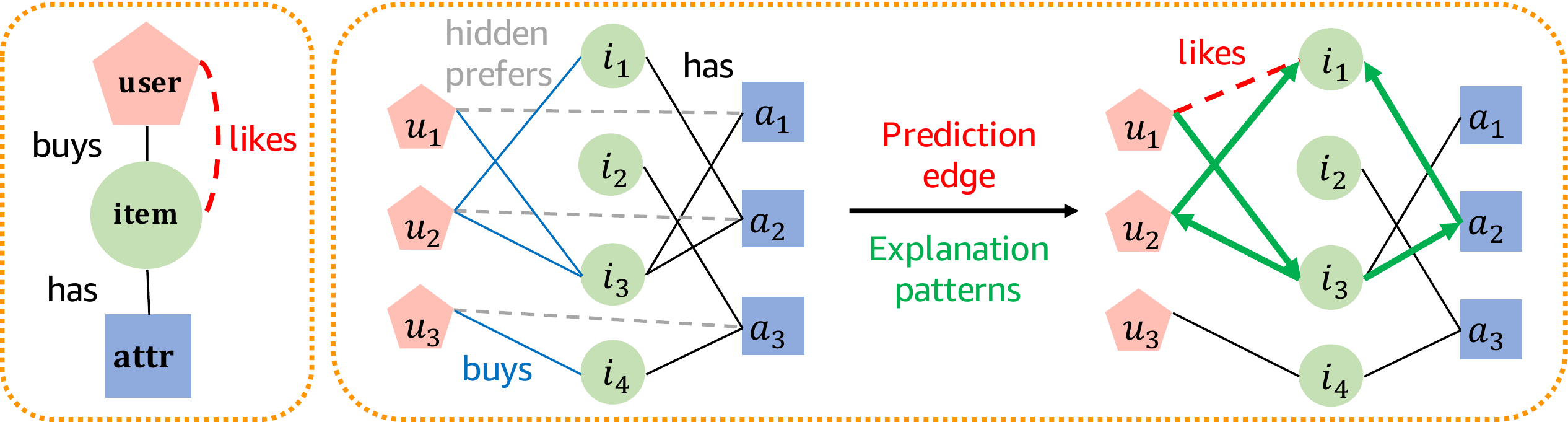}
\caption{Schema of \synthetic (the left box) and its generation process (the right box). Three types of base edges are generated first, i.e., ``has'' (black), ``hidden prefers'' ({\color{gray} dashed gray}), and ``buys'' ({\color{RoyalBlue} blue}). The solid ``has'' and ``buys'' edges are then used to generate ``likes'' edges ({\color{red} dashed red}) for prediction and the ground truth explanation patterns ({\color{ForestGreen} green arrows}).} 
\label{fig:create_graph_b}
\end{subfigure} 
\caption{The proposed augmented graph \aminer and the synthetic graph \synthetic.}
\label{fig:create_graph}
\Description{The proposed augmented graph has four types of nodes and three types of edges. The synthetic graph has three types of nodes and two types of edges. A new type ``likes'' is generated for both graphs and used for prediction.}
\end{figure*}

\paragraph{The augmented graph} \aminer is constructed by augmenting the AMiner citation network \cite{tang2008arnetminer}. A graph schema is shown in Figure~\ref{fig:create_graph_a}. The original AMiner graph contains four node types: author, paper, reference (ref), and field of study (fos), and edge types ``cites'', ``writes'', and ``in''. We construct \aminer by augmenting the original graph with new (author, paper) edges typed ``likes'' and define a paper recommendation task on \aminer for predicting the ``like'' edges. A new edge $(s, t)$ is augmented if there is at least one concise and informative path $p$ between them. In our augmentation process, we require the paths $p$ to have lengths shorter than a hyperparameter $l_{max}$ and with degrees of nodes on $p$ (excluding $s$ \& $t$) bounded by a hyperparameter $D_{max}$. We highlight these two hyperparameters because of the conciseness and informativeness principles discussed in Section~\ref{sec:problem}. The augmented edge $(s, t)$ is used for prediction. The ground truth explanation is the set of paths satisfying the two hyperparameter requirements. We only take the top $P_{max}$ paths with the smallest degree sums if there are many qualified paths. We train a GNN-based LP model to predict these new ``likes'' edges and evaluate explainers by comparing their output explanations with these path patterns as ground truth.

\paragraph{The synthetic graph} \synthetic is generated to mimic graphs with users, items, and attributes for recommendations. Figure~\ref{fig:create_graph_b} shows the graph schema and illustrates the generation process. We include three node types: ``user'', ``item'', and item attributes (``attr'') in the synthetic graph, and we build different types of edges step by step. Firstly, the ``has'' edges are created by randomly connecting items to attrs, and the ``hidden prefers'' edges are created by randomly connecting users to attrs. These edges represent items having attributes and user preferences for these attributes. Next, we randomly sample a set of items for each user, and we connect a (user, item) pair by a ``buys'' edge, if the user ``hidden prefers'' any attr the item ``has''. The ``hidden prefers'' edges correspond to an intermediate step for generating the observable ``buys'' edges. We remove the ``hidden prefers'' edges after ``buys'' edges are generated since we cannot observe `hidden prefers'' information in reality. An example of the rationale behind this generation step is that items have certain attributes, like the item ``ice cream'' with the attribute ``vanilla''. Then given that a user likes the attribute ``vanilla'' as hidden information, we observe that the user buys ``vanilla ice cream''. The next step is to generate more `buys'' edges between randomly picked (user, item) pairs if a similar user (two users with many shared item neighbors) buys this item. The idea is like collaborative filtering, which says similar users tend to buy similar items. The final step is generating edges for prediction and their corresponding ground truth explanations, which follows the same augmentation process described above for \aminer. For \synthetic, we have ``has'' and ``buys'' as base edges to construct the ground truth, 
and we create ``likes'' edges between users and items for prediction.

\begin{figure*}[t]
\centering
\includegraphics[width=\textwidth]{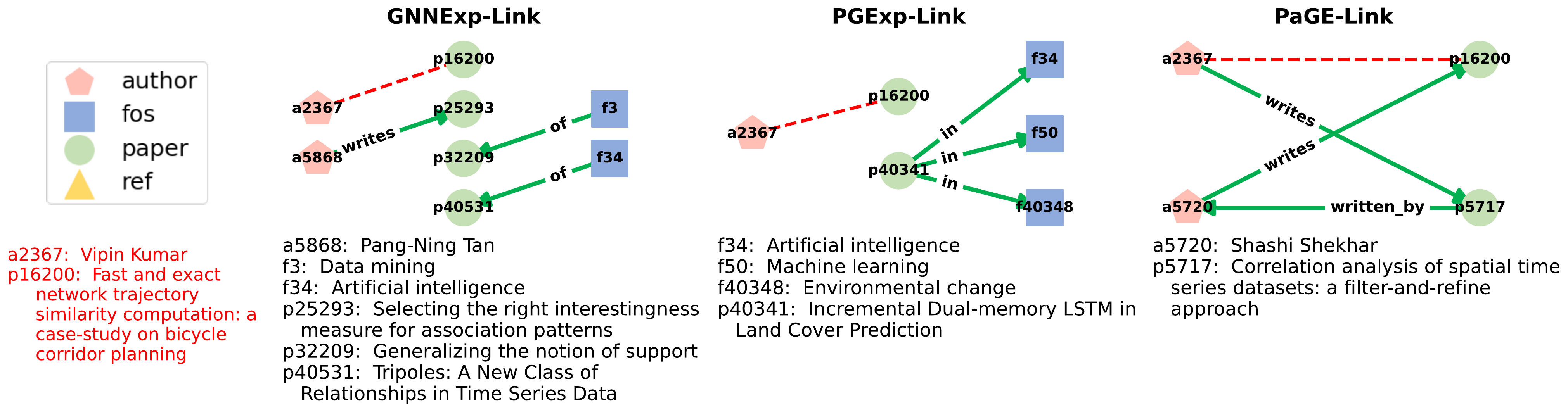}
\caption{Explanations ({\color{ForestGreen} green arrows}) by different explainers for the predicted link $(a2367, p16200)$ ({\color{red} dashed red}). \method explanation explains the recommendation by co-authorship, whereas baseline explanations are less interpretable.} 
\label{fig:case1}
\Description{PaGE-Link generates path explanation and the baselines fail to generate paths.}
\end{figure*}

\begin{figure}[t]
\centering
\includegraphics[width=\columnwidth]{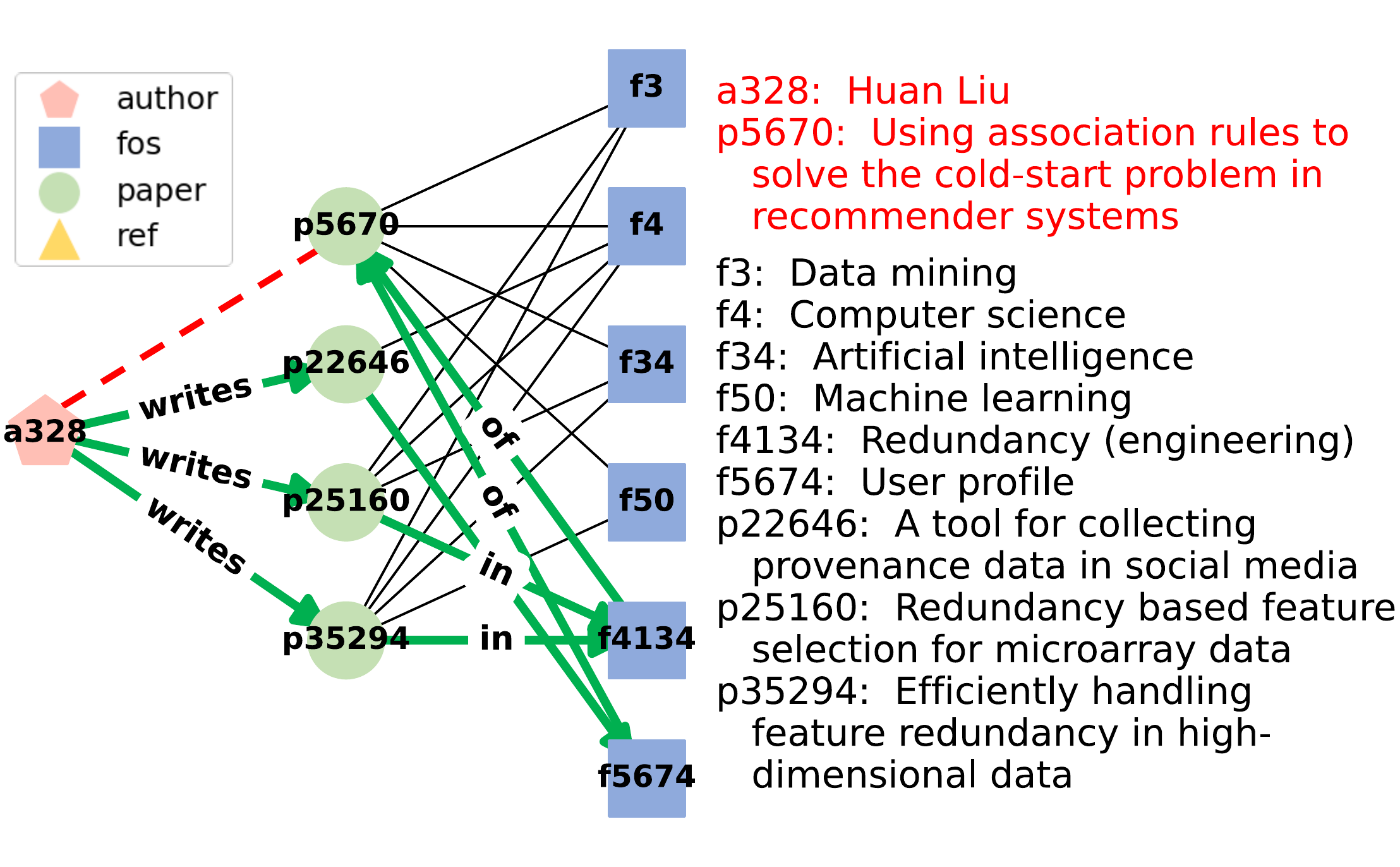}
\caption{Top three paths ({\color{ForestGreen} green arrows}) selected by \method for explaining the predicted link $(a328, p5670)$ ({\color{red} dashed red}). The selected paths are short and do not go through a generic field of study like ``Computer Science''.} 
\label{fig:case2}
\Description{PaGE-Link identifies concise and informative paths among many path choices.}
\end{figure}

\subsection{Experiment Settings}

\paragraph{The GNN-based LP model} As described in Section \ref{sec:preliminary}, the LP model involves a GNN encoder and a prediction head. We use RGCN \cite{rgcn} as the encoder to learn node representations on heterogeneous graphs and the inner product as the prediction head. We train the model using the cross-entropy loss. On each dataset, our prediction task covers one edge type $r$. We randomly split the observed edges of type $r$ into train:validation:test = 7:1:2 as positive samples and draw negative samples from the unobserved edges of type $r$. Edges of other types are used for GNN message passing but not prediction. 

\paragraph{Explainer baselines.} Existing GNN explanation methods cannot be directly applied to heterogeneous LP. Thus, we extend the popular GNNExplainer \cite{gnnexplainer} and PGExplainer \cite{pgexplainer} as our baselines. We re-implement a heterogeneous version of their mask matrix and mask predictor similar to the heterogeneous mask learning module in \method. For these baselines, we perform mask learning using their original objectives, and we generate edge-induced subgraph explanations from their learned mask. We refer to these two adapted explainers as GNNExp-Link and PGExp-Link below. We do not compare to other search-based explainers like SubgraphX~\cite{subgraphx} because of their high computational complexity (see Section \ref{subsec:complexity}). They work well on small graphs as in the original papers, but they are hard to scale to large and dense graphs we consider for LP.

\begin{table}[t]
\center
\small
\caption{ROC-AUC scores on learned masks. \method outperforms baselines. }
\begin{tabular}{@{}rcc|c@{}}
\toprule
& GNNExp-Link & PGExp-Link & \method (ours)  \\ \midrule    
\aminer & 0.829 & 0.586 & \textbf{0.928} \\ 
\synthetic & 0.608 & 0.578 & \textbf{0.954} \\ 
\bottomrule
\end{tabular}
\label{tab:auc}
\Description{PaGE-Link has better ROC-AUC scores than baselines.}
\end{table}

\subsection{Evaluation Results} \label{subsec:exp_results}

\paragraph{Quantitative evaluation.}
Both the ground truth and the final explanation output of \method are sets of paths. In contrast, the baseline explainers generate edge masks $\M$. For a fair comparison, we take the intermediate result \method learned, also the mask $\M$, and we follow \cite{pgexplainer} to compare explainers by their masks. Specifically for each computation graph, edges in the ground truth paths are treated as positive, and other edges are treated as negative. Then weights in $\M$ are treated as the prediction scores of edges and are evaluated with the ROC-AUC metric. A high ROC-AUC score reflects that edges in ground truth are precisely captured by the mask. The results are shown in Table \ref{tab:auc}, where \method outperforms both baseline explainers.

For scalability, we showed \method scales linearly in $O(|\E_c^k|)$ in Section \ref{subsec:complexity}. Here we evaluate its scalability empirically by generating ten synthetic graphs with various sizes from 20 to 5,500 edges in $\gG_c$. The results are shown in Figure~\ref{fig:crown_jewel2}b, which suggests the computation time scales linearly in the number of edges. 

\paragraph{Qualitative evaluation.}
A critical advantage of \method is that it generates path explanations, which can capture the connections between node pairs and enjoy better interpretability. In contrast, the top important edges found by baseline methods are often disconnected from the source, the target, or both, which makes their explanations hard for humans to interpret and investigate. We conduct case studies to visualize explanations generated by \method on the paper recommendation task on \aminer. 

Figure \ref{fig:case1} shows a case in which the model recommends the source author ``Vipin Kumar'' the recommended target paper titled ``Fast and exact network trajectory similarity computation: a case-study on bicycle corridor planning''. The top path explanation generated by \method goes through the coauthor ``Shashi Shekhar'', which explains the recommendation as Vipin Kumar and Shashi Shekhar coauthored the paper ``Correlation analysis of spatial time series datasets: a filter-and-refine approach'', and Shashi Shekhar wrote the recommended paper. Given the same budget of three edges, explanations generated by baselines are less interpretable.

Figure \ref{fig:case2} shows another example with the source author ``Huan Liu'' and the recommended target paper titled ``Using association rules to solve the cold-start problem in recommender systems''. \method generates paths going through the common fos of the recommended paper and three other papers written by Huan Liu: $p22646$, $p25160$, and $p35294$. We show the \method explanation with the top three paths in green. We also show other unselected fos shared by the $p22646$, $p25160$, and $p35294$ and the target paper. Note that the explanation paths all have length three, even though there are many paths with length five or longer, e.g., $(a328, p22646, f4, p25260, f4134, p5670)$. Also, the explanation paths go through the fos ``Redundancy (engineering)'' and ``User profile'' instead of generic fos like ``Artificial intelligence'' and ``Computer science''. This case demonstrates that explanation paths selected by \method are more concise and informative.

\section{Human Evaluation} \label{sec:human}
The ultimate goal of model explanation is to improve model transparency and help human decision-making. Human evaluation is thus the best way to evaluate the effectiveness of an explainer, which has been a standard evaluation approach in previous works \cite{grad-cam, lime, prince_2020}. 
We conduct a human evaluation by randomly picking 100 predicted links from the test set of \aminer and generate explanations for each link using GNNExp-Link, PGExp-Link, and \method. We design a survey with single-choice questions. In each question, we show respondents the predicted link and those three explanations with both the graph structure and the node/edge type information, similarly as in Figure \ref{fig:case1} but excluding method names. The survey is sent to people across graduate students, postdocs, engineers, research scientists, and professors, including people with and without background knowledge about GNNs. We ask respondents to ``please select the best explanation of \textit{`why the model predicts this author will like the recommended paper?'} ''. At least three answers from different people are collected for each question. In total, 340 evaluations are collected and 78.79\% of them selected explanations by \method as the best.

\section{Conclusion}\label{sec:conclusion}
In this work, we study model transparency and accountability on graphs. We investigate a new task: GNN explanation for heterogeneous LP. We identify three challenges for the task and propose a new path-based method, i.e. \method, that produces explanations with \textit{interpretable connections}, is \textit{scalable}, and handles graph \textit{heterogeneity}. \method explanations quantitatively improve ROC-AUC by 9 - 35\% over baselines and are chosen by 78.79\% responses as qualitatively more interpretable in human evaluation.

\begin{acks}
We thank Ziniu Hu for the helpful discussions on this work. This work is partially supported by NSF (2211557, 1937599, 2119643), NASA, SRC, Okawa Foundation Grant, Amazon Research Awards, Cisco Research Grant, Picsart Gifts, and Snapchat Gifts.
\end{acks}

\newpage
\bibliographystyle{ACM-Reference-Format}
\bibliography{reference.bib}

\appendix
\section{Proof of proposition \ref{prop:num_paths}} \label{app:num_paths_proof}
\begin{proof}
We prove $Z_{n,d} = o(S_{n,d})$ by definition, where we show $\lim_{n \to \infty} \frac{Z_{n,d}}{S_{n,d}} = 0$. As we can permute the indices of nodes in $\gG(n,d)$, without loss of generality, we assume $Z_{n,d}$ is the expected number of paths between nodes indexed 1 and n. Our proof is mainly based on the result in \cite{s-t-path}, which computes the expected number of all 1-n paths, i.e., $Z_{n,d} = (n - 2)! d^{n-1} e(1 + o(1))$. On the other hand, the number of edge-induced subgraphs considered in \cite{gnnexplainer, pgexplainer} equals the size of the power set of all edges, i.e., $S_{n,d} = 2^{d\binom{n}{2}}$. We thus have 

\begin{align}\setcounter{equation}{0}
\log Z_{n,d}
&= \log \left[(n - 2)! d^{n-1} e(1 + o(1)) \right]\\
&< \log \left[\sqrt{2 \pi (n-2)}(\frac{n-2}{e})^{(n-2)}e^{\frac{1}{12(n-2)}} d^{n-1} e(1 + o(1)) \right]\\
&= \frac{1}{2}\log(2 \pi (n-2)) + (n-2)\log(\frac{n-2}{e}) + \log \frac{1}{12(n-2)} \notag\\
&\quad + (n - 1) \log d + 1 + \log(1 + o(1)) \\
&= O(\log n) + O(n \log n) + O(\log \frac{1}{n}) + O(n\log d) \\
&\quad + \log(1 + o(1)) \\
&= O(n \log n) + \log(1 + o(1)) \\
\log S_{n,d}
&= \log 2^{d\binom{n}{2}} = {d\binom{n}{2}} \log 2 = O(n^2) \\
\lim_{n \to \infty} \frac{Z_{n,d}}{S_{n,d}}
&= \lim_{n \to \infty} \exp(\log{\frac{Z_{n,d}}{S_{n,d}}}) \\
&= \exp(\lim_{n \to \infty} \log{\frac{Z_{n,d}}{S_{n,d}}}) \\
&= \exp(\lim_{n \to \infty} \log Z_{n,d} - \log S_{n,d}) \\
&= \exp(\lim_{n \to \infty} O(n \log n) + \log(1 + o(1)) - O(n^2)) \\
&= 0
\end{align}
Step (1) to (2) is Stirling's formula. Step (8) to (9) is because $\exp$ is continuous.
\end{proof}

\section{Detailed theorem \ref{thm:kcore_size}} \label{app:kcore_size}
We now state a more detailed version of Theorem \ref{thm:kcore_size}. This theorem gives the exact formula of $\delta_{\V}(n, d, k)$ and $\delta_{\E}(n, d, k)$, which are built upon a Poisson random variable. The argument is adapted from \cite{kcore-asymptotic, kcore-emergence}. Readers can refer to \cite{kcore-asymptotic, kcore-emergence} for the proof.

For $\mu > 0$, let $Po(\mu)$ denote a Poisson distribution with mean $\mu$. Let $\psi_k(dn) = P(Po(dn) \geq k)$ be the tail probability of $Po(dn)$. Let $c_k = \inf_{\mu > 0} \mu / \phi_{k-1}(\mu)$. When $dn > c_k$, the equation $\mu / \psi_{k-1}(\mu) = dn$ will have two roots for $\mu$. Let $\mu(dn, k)$ be the larger root. Then we have the following more detailed version of Theorem \ref{thm:kcore_size} with $\delta_{\V}(n, d, k)$ and $\delta_{\E}(n, d, k)$ as functions of $\mu(dn, k)$.

\begin{theorem}[Pittel, Spencer and Wormald] \label{thm:kcore_size_detail}
Let $\gG(n, d)$ be a random graph with $m$ edges as in Proposition \ref{prop:num_paths}. Let $\gG^k(n, d) = (\V^k(n, d), \E^k(n, d))$ be the k-core of $\gG(n, d)$. When $dn > c_k$, $\gG^k(n, d)$ will be nonempty with high probability (w.h.p.) for large n. Also, $\gG^k(n, d)$ will contain $\psi_k(\mu(dn, k)) n$ nodes and $[\mu(dn, k)^2 / (d^2n(n-1))]m$ edges w.h.p. for large n, i.e., $|\V^k(n, d)| / n \xrightarrow{p} \psi_k(\mu(dn, k))$ and $|\E^k(n, d)| / m \xrightarrow{p} \mu(dn, k)^2 / (d^2n(n-1))$ ($\xrightarrow{p}$ stands for convergence in probability).
\end{theorem}

\section{Complexity of SubgraphX} \label{app:complexity}
The search-based methods often have much higher time complexity exponential in the number of nodes or edges. Thus, a budget is forced instead of searching subgraphs with all sizes. For example, SubgraphX finds all connected subgraphs with at most $B_{node}$ nodes, which has complexity $\Theta(|\V_c|\hat D^{2B_{node} - 2})$ for a graph with maximum degree $\hat D = \max_{v \in \V} D_v$. This complexity can be shown using the following two lemmas.

\begin{lemma} \label{lemma:complexity}
For a graph $\gG$ with n vertices, the number of the connected subgraph of $\gG$ having $B_{node}$ nodes is bounded below by the number of trees in $\gG$ having $B_{node}$ nodes.
\end{lemma}
\begin{proof}
Each connected subgraph has a spanning tree.    
\end{proof}

\begin{lemma} \label{lemma:complexity2}
For a graph $\gG$ with node set $\V$, the number of trees in $\gG$ having $B_{node}$ tree nodes is $\Theta(|\V| \hat D^{2B_{node} - 2})$.
\end{lemma}

\begin{proof}
    See \cite{numsubgraphs} for proof using an encoding procedure. 
\end{proof}

\section{Dataset Details} \label{app:data}
We show the hyperparameters for constructing the datasets in Section \ref{sec:experiment} in Table \ref{tab:data_hyper}, which includes the augmentation of the Aminer citation graph and the generation of the synthetic graph.

\begin{table}[h]
\center
\small
\caption{Hyperparameters for constructing \aminer and \synthetic}
\begin{tabular}{@{}rccc@{}}
\toprule
& $l_{max}$ & $D_{max}$ & $P_{max}$ \\ \midrule    
\aminer & 3 & 30 & 5 \\ 
\synthetic & 3 & 15 & 5 \\ 
\bottomrule
\end{tabular}
\label{tab:data_hyper}
\end{table}

\section{Path Hit Evaluation}
Besides ROC-AUC scores, another way to evaluate the explanations is through the path hit rate (HR). Specifically, we fix the budget of $B$ edges and evaluate whether an explanation can hit any complete path in the ground truth. Note that the ground truth for each link $(s, t)$ only has the top $P_{max}$ shortest paths with the smallest degree sums, so hitting a long path or a less informative path with high-degree generic nodes will not count. 

For a fair comparison with baselines, we take the generated explanation mask $\M$ for each method, select the top $B$ weighted edges to compare against the ground truth. We show results with different budget $B$ in Table \ref{tab:hit_rate}. Explanations generated by \method have higher path HR than baselines on both datasets. In contrast, GNNExp-Link and PGExp-Link can barely hit any path in the ground truth for $B$ less than 50.

Note that the actual explanation output of \method is a set of paths $\P$. If we evaluate $\P$ instead of the top cut of the intermediate output mask $\M$. Then \method can achieve perfect path HR (=1) when the budget $|\P|$ gets large. 

\begin{table}[th]
\center
\small
\caption{Path hit rate (HR). \method has high HR with a small budget $B$. Baselines achieve nonzero HR for large $B$.}
\begin{tabular}{@{}r|r|cc|c@{}}
\toprule
& B & GNNExp-Link & PGExp-Link & \method (ours)  \\ \midrule    
\multirow{4}{*}{\aminer} 
    & 10  & 0.000 & 0.000 & \textbf{0.007}  \\ 
    & 50  & 0.002 & 0.000 & \textbf{0.194}  \\ 
    & 100  & 0.019 & 0.000 & \textbf{0.425}  \\ 
    & 200 & 0.064 & 0.002 & \textbf{0.645}  \\  \midrule
\multirow{4}{*}{\synthetic}
    & 10  & 0.000 & 0.000 & \textbf{0.163}  \\ 
    & 50  & 0.008 & 0.032 & \textbf{0.705}  \\ 
    & 100  & 0.016 & 0.039 & \textbf{0.790}  \\ 
    & 200 & 0.046 & 0.101 & \textbf{0.907}  \\ 
\bottomrule
\end{tabular}
\label{tab:hit_rate}
\Description{PaGE-Link has better path hit rate than baselines.}
\end{table}

\end{document}